\documentclass[letterpaper, 10 pt, conference]{ieeeconf}  

\IEEEoverridecommandlockouts                              
\usepackage{textcomp}

\usepackage{sansmath}
\usepackage{graphicx}
\usepackage{amsmath}
\usepackage{xspace}
\usepackage{multirow}
\usepackage{diagbox}
\usepackage{amsfonts}
\usepackage{mathtools}
\usepackage{hyperref}
\usepackage[english]{babel}
\usepackage[version=4]{mhchem}
\usepackage{siunitx}
\usepackage{blindtext}
\usepackage{hyperref}
\usepackage{longtable,tabularx}
\usepackage[utf8]{inputenc}
\usepackage{graphicx}
\usepackage{subcaption}
\usepackage[version=4]{mhchem}
\usepackage{longtable,tabularx}
\usepackage{algorithm}
\usepackage{algpseudocode}
\usepackage{float}
\usepackage{titlesec}
\usepackage{tikz}
\usepackage{pgfplots}
\pgfplotsset{compat=1.17}
\usepackage{booktabs}
\usepackage{pgfplotstable} 
\usetikzlibrary{positioning, shapes.geometric, fit,arrows.meta, backgrounds}
\usetikzlibrary{calc}
\newcommand{\bd}{{\boldsymbol{d}}}

\newcommand{\bx}{{\boldsymbol{x}}}
\newcommand{\bu}{{\boldsymbol{u}}}
\newcommand{\by}{{\boldsymbol{y}}}

\newcommand{\ixinactive}{\mathcal{I}^\bx_{\text{inactive}}}

\newcommand{\iuinactive}{\mathcal{I}^\bu_{\text{inactive}}}
\newcommand{\ixactive}{\mathcal{I}^\bx_{\text{active}}}
\newcommand{\iuactive}{\mathcal{I}^\bu_{\text{active}}}

\usepackage{xspace}
\newcommand{\ours}{TransformerMPC\xspace}

\usepackage[compress]{cite}

\pgfkeys{
    /cites/refA/.initial = {Ours},
    /cites/refB/.initial = {\cite{stellato2020osqp_al1}},
    /cites/refC/.initial = {\cite{bishop2024relu_qp}},
    /cites/refD/.initial = {\cite{bambade2022prox_al2}},
    /cites/refE/.initial = {\cite{ferreau2014qpoases_as1}},
}

\newtheorem{remark}{\textnormal{\textbf{Remark}}}

\newtheorem{problem}{\bf Problem}

\usepackage{dirtytalk}

\newtheorem{lemma}{\bf Lemma}

\newcommand{\mcl}[1]{\mathcal{#1}}

\DeclareMathOperator*{\minimize}{minimize}

\newcommand{\subjectto}{\mathop{\rm subject~to}}

\title{\LARGE \bf
TransformerMPC: Accelerating Model Predictive Control via Transformers
}

\author{Vrushabh Zinage$^{1}$, Ahmed Khalil$^{1}$,  Efstathios Bakolas $^{1}$
\thanks{$^{1}$Vrushabh Zinage, Ahmed Khalil, Efstathios Bakolas are with the Department of Aerospace Engineering and Engineering Mechanics, University of Texas at Austin
        {\tt\small vrushabh.zinage@utexas.edu,akhalil@utexas.edu, bakolas@austin.utexas.edu}}%
}

\begin{document}
\bibliographystyle{IEEEtran}

\maketitle
\thispagestyle{empty}
\pagestyle{empty}

\begin{abstract}

In this paper, we address the problem of reducing the computational burden of Model Predictive Control (MPC) for real-time robotic applications. We propose TransformerMPC, a method that enhances the computational efficiency of MPC algorithms by leveraging the attention mechanism in transformers for both online constraint removal and better warm start initialization. Specifically, TransformerMPC accelerates the computation of optimal control inputs by selecting only the active constraints to be included in the MPC problem, while simultaneously providing a warm start to the optimization process. This approach ensures that the original constraints are satisfied at optimality. TransformerMPC is designed to be seamlessly integrated with any MPC solver, irrespective of its implementation. To guarantee constraint satisfaction after removing inactive constraints, we perform an offline verification to ensure that the optimal control inputs generated by the MPC solver meet all constraints. The effectiveness of TransformerMPC is demonstrated through extensive numerical simulations on complex robotic systems, achieving up to $35\times$ improvement in runtime without any loss in performance. Videos and code are available at this website\footnote{Website: \url{https://transformer-mpc.github.io/}}\footnote{Code will be made available after final submission}.
\end{abstract}

\section{Introduction}
We consider the problem of improving the computational efficiency of Model Predictive Control (MPC) algorithms for general nonlinear systems under non-convex constraints. MPC is one of the most popular frameworks used in robotic systems for embedded optimal control with constraints, enabling them to operate autonomously in various real-world situations with applications ranging from legged and humanoid robots \cite{katayama2023model_legged_mpc1,hong2020real_legged_mpc2} to quadrotors \cite{didier2021robust_quad_mpc1,nguyen2024tiny_mpc}, swarms of spacecraft \cite{morgan2014model_swarms_mpc}, and ground robots \cite{nascimento2018nonholonomic_ground_mpc1,yu2021model_ground_mpc2}, to name but a few. However, they are generally computationally expensive as their implementation relies on solving constrained optimal control problems (OCPs). Well-established MPC solvers are usually computationally efficient but often restrict themselves to OCPs with convex quadratic cost functions subject to linear dynamics and constraints giving rise to (constrained) quadratic programs (QPs) \cite{stellato2020osqp_al1,bishop2024relu_qp,bambade2022prox_al2,ferreau2014qpoases_as1}. On the other hand, nonlinear MPC solvers \cite{fiedler2023dompc_nmpc3,Andersson2019_casadi_nmpc2,azhmyakov2008convex_nmpc_1} can handle general nonlinear dynamics with both convex and non-convex constraints but are generally computationally expensive.

Many approaches for reducing the computational burden of MPC have been proposed in the literature \cite{bishop2024relu_qp,howell2019altro}. Most of these methods, however, focus on the computational aspects of solving the underlying optimization problems. 
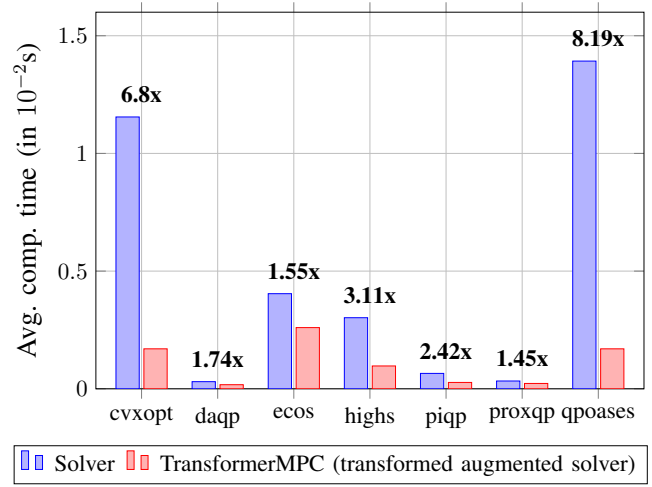
\begin{figure}[]
    \centering
\begin{tikzpicture}[scale=0.88]
\begin{axis}[
    ybar,
    ymin=0,
    ymax=1.6,
        legend style={
        at={(-0.15,-0.15)},
        anchor=north west,
        legend columns=-1,
        column sep=0.3em 
    },
    ylabel={\large Avg. comp. time (in $10^{-2}$s)},
    symbolic x coords={ cvxopt, daqp, ecos, highs, piqp, proxqp, qpoases},
    xtick=data,
         x=1.15cm, 
             grid=both, 
    ]
\addplot coordinates {(cvxopt,1.1549) (daqp,0.0303) (ecos,0.4041) 
                      (highs,0.3020)  (piqp,0.0654) (proxqp,0.0331) 
                      (qpoases,1.3923) };
\addplot coordinates {(cvxopt,0.1698) (daqp,0.0174) (ecos,0.26) 
                      (highs,0.097) (piqp,0.027) (proxqp,0.0227) 
                      (qpoases,0.1698) };
\node[above, yshift=-0.4cm, line width=1cm] at (axis cs:cvxopt,1.1549) {\textbf{\textcolor{black}{ 6.8x}}};
\node[above, yshift=-0.4cm, line width=1cm] at (axis cs:daqp,0.0303) {\textbf{\textcolor{black}{ 1.74x}}};
\node[above, yshift=-0.4cm, line width=1cm] at (axis cs:ecos,0.4041) {\textbf{\textcolor{black}{ 1.55x}}};

\node[above, yshift=-0.4cm, line width=1cm] at (axis cs:highs,0.3020) {\textbf{\textcolor{black}{ 3.11x}}};

\node[above, yshift=-0.4cm, line width=1cm] at (axis cs:piqp,0.0654) {\textbf{\textcolor{black}{ 2.42x}}};

\node[above, yshift=-0.4cm, line width=1cm] at (axis cs:proxqp,0.0331) {\textbf{\textcolor{black}{ 1.45x}}};
\node[above, yshift=-0.4cm, line width=1cm] at (axis cs:qpoases,1.3923) {\textbf{\textcolor{black}{ 8.19x}}};

\legend{Solver, TransformerMPC (transformed augmented solver)}
\end{axis}
\end{tikzpicture}



    \caption{\small Average computational time (in $10^{-2}$s) for various solvers on Upkie wheeled biped robot with and without \ours augmentation, demonstrating significant reductions in computational time of \ours in accelerating solver performance. }
    \centering
    \label{fig:random_mpc}
\end{figure}
\begin{figure*}[h!]
    \centering
    \begin{minipage}[t]{1\textwidth}
        \centering
\input{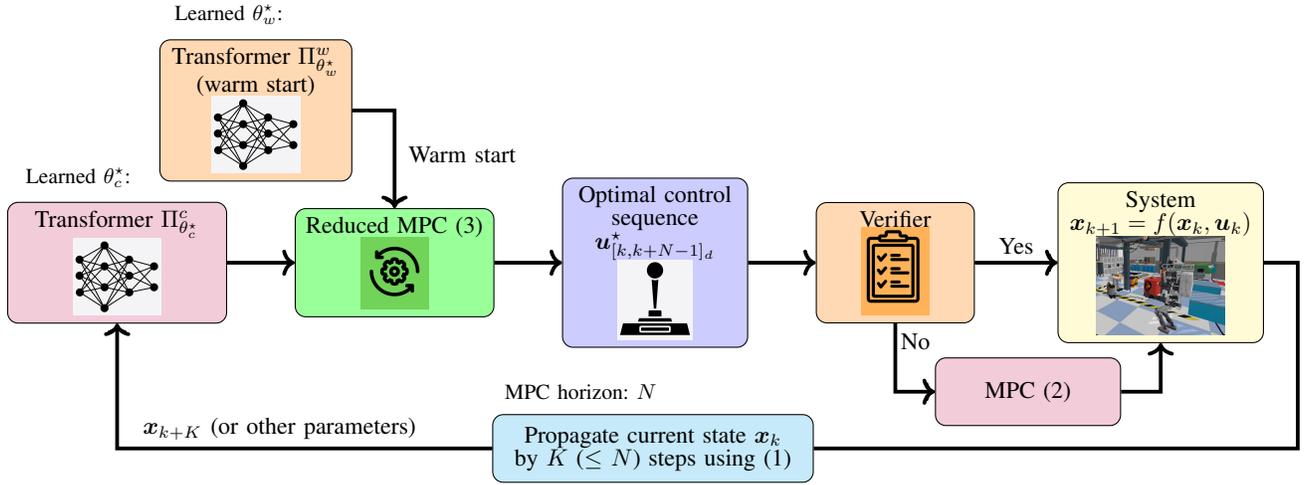}
        \subcaption{\small Our proposed approach (\ours)}
\label{fig:proposed_approach}
    \end{minipage}
    \caption{\small \ours improves the computational efficiency of the MPC framework by incorporating a transformer-based attention mechanism to determine active constraints as well as better warm start initialization, which can be integrated with any state-of-the-art MPC solver. Additionally, we have a verifier step that checks whether the optimal control sequence synthesized using \eqref{eqn:simplified_mpc_problem} satisfies all the constraints of the original MPC \eqref{eqn:mpc_typical}. Note that \ours is applicable to general nonlinear MPC problems as well.}
    \label{fig:combined_approach_training}
\end{figure*}
Most MPC solvers are generally classified into three categories: interior-point methods, augmented-Lagrangian / ADMM methods, and active-set methods. Interior-point solvers \cite{aps2019mosek_ip1,gurobi2023gurobi_ip2,frison2020hpipm_ip3,goulart2024clarabel_ip4,andersen2020cvxopt_ip5,domahidi2013ecos_ip6,pandala2019qpswift_ip7,schwan2023piqp_ip8,wachter2006implementation_ipopt} offer robust convergence, however, are usually challenging to warm start, making them less ideal for MPC problems. Augmented-Lagrangian and ADMM solvers \cite{stellato2020osqp_al1,bambade2022prox_al2,hermans2022qpalm_al3,howell2019altro, nguyen2024tiny_mpc} are prevalent in MPC solvers for robotic applications due to their fast convergence. However, both of these classes of methods are computationally efficient when restricted to quadratic programs (QPs) and usually do not scale well when applied to general nonlinear systems with non-convex constraints. By contrast, nonlinear MPC \cite{azhmyakov2008convex_nmpc_1,lautenschlager2015convexity_nmpc_2}, which accounts for nonlinear dynamics and constraints, typically results in non-convex optimization problems, making it more challenging to find an optimal solution. Active-set methods \cite{ferreau2014qpoases_as1,goldfarb1983numerically_quad_prog_as2,hall2023highs_as3,arnstrom2022dual_as4} which focus on constraint removal to improve the computational efficiency allow for easy warm starts and can be fast if the active set is correctly identified, though they suffer from combinatorial worst-case time complexity \cite{nocedal1999numerical} and are usually restricted to QPs. On the other hand, explicit MPC based methods pre-compute optimal control laws offline, enabling faster online implementation \cite{bemporad2002explicit_mpc_1}. However, the lookup tables required in explicit MPC can become exponentially large for systems with more than a few states and inputs \cite{alessio2009survey_exp_mpc_2}.

Recently, a new line of research \cite{shi2019neural_lander,o2022neural_fly,briden2024improving, salzmann2023real_time_mpc,guffanti2024transformers_rendezvous, klauvco2019machine_ml_warm_start,schwenkel2020online_ml_warm_start2, chen2022large} has brought together the best of both classical optimization-based methods and learning-based methods to rapidly and efficiently generate control inputs for OCPs. The transition towards incorporating learning-based methods into robotic systems is motivated by three main factors. First, the computational demands of deploying trained machine learning models during inference are minimal and likely align with the limited computational resources available on most robotic systems. Second, learning-based methods hold potential for addressing control problems that involve complex, multi-stage processes {with} potentially non-convex cost functions. Finally, learning-based methods can continuously refine the model of the system by incorporating new data over time, leading to improved control performance.

Inspired by these approaches, our approach uses the attention mechanism \cite{vaswani2017attention} from transformers to identify active constraints as well as compute better warm start initial conditions, enabling integration with existing MPC solvers \cite{aps2019mosek_ip1,gurobi2023gurobi_ip2,frison2020hpipm_ip3,goulart2024clarabel_ip4,andersen2020cvxopt_ip5,domahidi2013ecos_ip6,pandala2019qpswift_ip7,schwan2023piqp_ip8,wachter2006implementation_ipopt} or learning-based optimal controllers \cite{shi2019neural_lander,o2022neural_fly,briden2024improving,salzmann2023real_time_mpc,guffanti2024transformers_rendezvous}. Unlike most methods that are specific to QPs, our approach applies to general nonlinear MPC problems as well.
The contributions of this paper are as follows:
\begin{enumerate}
    \item We introduce \ours, a method that utilizes the attention mechanism in transformers for efficient online inactive constraint removal in MPC as well as for warm starting. This accelerates the computation of optimal control inputs by selecting a subset of constraints that are active at optimality in the MPC problem while ensuring that the solution maintains the original constraint satisfaction properties.
    \item Our approach is agnostic to the specific implementation of the optimization problem, allowing it to be integrated with any state-of-the-art MPC solver. Furthermore, for {QP-based MPC problems, we demonstrate how the transformer's active constraint predictions can be used to solve the MPC problem analytically.}
    \item To ensure constraint satisfaction after \ours removes inactive constraints, we perform offline verification to ensure that the optimal control sequence computed by the MPC solver post-removal satisfies all the constraints.
    \item We further enhance the computational efficiency for nonlinear MPC by combining multiple nonlinear / non-convex constraints into a single smooth constraint using log-sum-exp functions, accelerated via GPU parallelization of the summation operation.
\end{enumerate}

The paper is organized as follows. Section \ref{sec:prelim_and_problem} discusses the preliminaries and the problem statement followed by our proposed approach in Section \ref{sec:proposed_approach}. Finally, we discuss the results in Section \ref{sec:results} followed by some concluding remarks in Section \ref{sec:conclusion}.

\section{Preliminaries and Problem Formulation\label{sec:prelim_and_problem}}
Consider the discrete-time system given by
\begin{align}
    \bx_{k+1}=f(\bx_k,\bu_k),\quad\quad \bx_0=\bx^0,
    \label{eqn:discrete_nonlinear_system}
\end{align}
where $\bx_k\in\mathcal{X}\subset\mathbb{R}^n$ and $\bu_k\in\mathcal{U}\subset\mathbb{R}^m$ are the state and control input at time step $k$ respectively, $\bx^0$ is the initial state, $f:\mathcal{X}\times\mathcal{U}\rightarrow\mathcal{X}$ is continuously differentiable, and $\mathcal{X}$ and $\mathcal{U}$ are compact sets. The receding horizon MPC problem can be formulated as:
\begin{subequations}
    \begin{align}
    &\bu^\star_{[k,k+N-1]_d}=\\
    &\underset{\bu_k, \ldots, \bu_{k+N-1}}{\mathrm{argmin}}  \;\sum_{i=0}^{N-1} \ell(\bx_{k+i}, \bu_{k+i})\nonumber \\
   & \subjectto\;\bx_{k+i+1} = f(\bx_{k+i}, \bu_{k+i}), \\ 
    &\quad\quad\quad\quad\quad\bx_{k+i} \in \mathcal{X}, \quad \forall\; i = [0,N-1]_d,\\
    &\quad\quad\quad\quad\quad\bu_{k+i} \in \mathcal{U}, \quad \forall\; i = [0,N-1]_d,\\
    &\quad\quad\quad\quad\quad\bx_{k+N} \in \mathcal{X}_f,
\end{align}
\label{eqn:mpc_typical}
\end{subequations}
where $\bu^\star_{[k,k+N-1]_d}=[\bu_k^\star,\dots,\bu_{k+N-1}^\star]^\mathrm{T}$ is the optimal control sequence, $N$ is the time horizon, $\ell(\bx_k, \bu_k)$ is the stage cost function, $\mathcal{X}$ and $\mathcal{U}$ are the feasible sets for the state and control input, respectively, and $\mathcal{X}_f$ is the terminal constraint set. We assume that the sets $\mathcal{X}$ and $\mathcal{U}$ are characterized by sequences of smooth functions as $\mathcal{X}:=\{\bx \mid g_j(\bx)\leq 0,\forall~j\in[0,N_\mathcal{X}]_d\}$ and $\mathcal{U}:=\{\bu \mid h_j(\bu)\leq 0,\forall\;j\in[0,N_\mathcal{U}]_d\}$, respectively.{The notation $[i, j]_d$ $(i\geq j)$ represents the set of integers $\{i,i+1,\dots,j\}$.}
A constraint $g_j(\bx) \leq 0$ or $h_j(\bu) \leq 0$ is considered active at a given time step if the equality holds at optimality, i.e., $g_j(\bx^\star) = 0$ or $h_j(\bu^\star) = 0$. A constraint is active if it directly affects the current optimization solution, potentially limiting the feasible region of the state or control input. Conversely, a constraint is considered inactive if it is strictly satisfied, i.e., $g_j(\bx^\star) < 0$ or $h_j(\bu^\star) < 0$. Inactive constraints do not restrict the current solution and, therefore, have no immediate impact on the feasible region at that time step.
Consequently, the active constraint set $\mathcal{C}_{\text{active}}$ is defined as the set of constraints where the equality holds at a given time step, i.e., $\mathcal{C}_{\text{active}} = \left\{g_j(\bx^\star) = 0 \mid j \in [0, N_\mathcal{X}]_d\} \cup \{h_j(\bu^\star) = 0 \mid j \in [0, N_\mathcal{U}]_d\right\}$. These constraints restrict the feasible set of states or control inputs. The inactive constraint set $\mathcal{C}_{\text{inactive}}$ includes constraints where strict inequalities hold, i.e., $\mathcal{C}_{\text{inactive}} = \left\{g_j(\bx^\star) < 0 \mid j \in [0, N_\mathcal{X}]_d\} \cup \{h_j(\bu^\star) < 0 \mid j \in [0, N_\mathcal{U}]_d\right\}$, and do not impact the optimization at that time step. Furthermore, let the indices corresponding to inactivate state constraints, inactive input constraints, active state constraints, and active input constraints be given by $\ixinactive=\{j \mid g_j(\bx^\star) < 0\}$, $\iuinactive=\{j \mid h_j(\bu^\star) < 0\}$, $\ixactive=\{j \mid g_j(\bx^\star) = 0\}$ and $\iuactive=\{j \mid h_j(\bu^\star) = 0\}$, respectively.
At each time step $k$, the first control input $\bu_k^\star$ from the optimal sequence {$\bu^\star_{[k,k+N-1]_d}$} is applied to \eqref{eqn:discrete_nonlinear_system}, and the problem is solved again at the next time step. 

{

\subsection{Problem Statement\label{subsec:problem_statement}}
We now {formally state} the problem statement:
{\begin{problem}
    \normalfont Given the system dynamics \eqref{eqn:discrete_nonlinear_system}, and parameters of the MPC problem instance that uniquely characterize the optimal solution, the goal is to predict the set of active constraints at optimality.
    \label{prob:problem_statement}
\end{problem}
Note that predicting active and inactive constraints is generally non-trivial, especially for nonlinear MPC problems, as it requires one to know the optimal solution to \eqref{eqn:mpc_typical}. In addition, most active set methods \cite{ferreau2014qpoases_as1,goldfarb1983numerically_quad_prog_as2,hall2023highs_as3,arnstrom2022dual_as4} from the literature focus on constraint removal, but are mainly restricted to QPs with linear constraints and suffer from combinatorial worst-case time complexity \cite{nocedal1999numerical}. Note that problem \ref{prob:problem_statement} considers general MPC problems with nonlinear dynamics and non-convex/convex constraints.

\section{Proposed Approach \label{sec:proposed_approach}}
In this section, we first present the reduced MPC problem \eqref{eqn:simplified_mpc_problem} and show that the optimal control sequence synthesized by solving \eqref{eqn:simplified_mpc_problem} is equivalent to solving the original MPC problem \eqref{eqn:mpc_typical}. We next discuss our proposed approach \ours. 
The overall approach involves two main phases: a learning phase and an execution phase. In the learning phase, the transformer is trained on a dataset containing various MPC problem instances \eqref{eqn:mpc_typical}, where each instance is characterized by parameters such as initial conditions $ \bx_0 $, reference trajectories $ \{\bx^{\textbf{ref}}_k,\dots,\bx^{\textbf{ref}}_{k+N}\} $ etc. The transformer learns to map these {parameters} to the corresponding active constraints $ \mathcal{C}_\text{active} $ (Section \ref{subsec:learning_framework}). Another transformer model is used for better warm start initialization that predicts a control sequence close to the optimal control sequence (Section \ref{subsec:transformer_warm_start}). In the execution phase, these trained models are used for real-time prediction of the active constraints for a given set of parameters as well as for better warm start initialization (Section \ref{subsec:excetuation_phase}). {This allows the MPC solver to ignore all the inactive constraints and focus solely on the active constraints, thereby reducing the computational burden while ensuring that the control objectives are met.} 
Next, we consider the special case of linear MPC problems where the costs are quadratic, and the dynamics and constraints are linear, thereby reducing the problem to a QP. Given the transformer's constraint predictions, we show that the linear MPC problem can be solved analytically (Section \ref{subsec:qp}). Finally, we propose an approach to further accelerate nonlinear MPC problems by combining multiple non-convex constraints into a single smooth constraint using log-sum-exp functions, with GPU parallelization thereby enhancing computational efficiency (Section \ref{subsec:acc_nonlinear_mpc}).
\begin{figure}[ht]
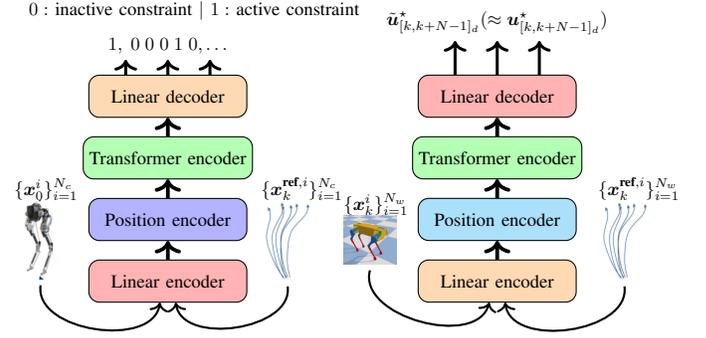

    \centering
    \begin{minipage}[t]{0.24\textwidth}
        \centering
        \input{tikz/training}
         \subcaption{\small $\Pi_{\theta_c}^c$ (Inactive constraint removal)}
        \label{fig:training_c}
    \end{minipage}%
    \hfill
        \begin{minipage}[t]{0.24\textwidth}
        \centering
        \input{tikz/training_warm_start}
         \subcaption{\small $\Pi_{\theta_w}^w$ (Warm start)}
        \label{fig:training_w}
    \end{minipage}%
    \caption{{
    For training both transformers, the input corresponds to problem parameters, such as the initial state $\bx_0$ and reference trajectories $\bx^{\textbf{ref}}$, that uniquely characterize the optimal solution. The transformer for inactive constraint removal, $\Pi^c_{\theta_c}$, returns the set of active constraints at optimality, while the transformer for warm starting, $\Pi^w_{\theta_w}$ returns a better initial guess for an optimal control sequence.}}
    \label{fig:training}
\end{figure}
\subsection{Constraint removal and simplified MPC problem\label{subsec:simplified_mpc_problem}}
We first reformulate the original MPC optimal control problem \eqref{eqn:mpc_typical} into a more reduced and simplified problem by eliminating constraints that are identified as inactive at the optimal solution. Note that an inactive constraint in  \eqref{eqn:mpc_typical} does not affect the optimal solution, meaning the optimal control sequence $\bu^\star_{[k,k+N-1]_d}$ from \eqref{eqn:mpc_typical} remains unchanged if the constraint is excluded from \eqref{eqn:mpc_typical}. 

\begin{lemma}
\normalfont  Let $\bx^0 \in \mathcal{X}$ be an {arbitrary} initial condition. {Assume that $\ixactive,\;\iuactive$ and $\ixinactive,\;\iuinactive$ represent the indices of the active and inactive constraints of \ref{eqn:mpc_typical}, respectively}. Then, the optimal control sequence $\bu^\star_{[k,k+N-1]_d}$ is a solution to 
\begin{subequations}
    \begin{align}
    &\underset{\bu_k, \ldots, \bu_{k+N-1}}{\minimize}  &&\sum_{i=0}^{N-1} \ell(\bx_{k+i}, \bu_{k+i}) \\
    &\subjectto && \bx_{k+i+1} = f(\bx_{k+i}, \bu_{k+i}), &i = [0,N-1]_d, \\
    &&& g_j(\bx)\leq 0, \quad\forall j\in \ixactive, \label{eqn:nonlin_cons_1}\\
    &&& h_j(\bu)\leq 0, \quad\forall j\in \iuactive,\label{eqn:nonlin_cons_2} 
\end{align}
\label{eqn:simplified_mpc_problem}
\end{subequations}
if and only if it is also a solution to problem \eqref{eqn:mpc_typical}.
\label{lemma:1}
\end{lemma}
\begin{proof}
    \normalfont
The necessary conditions for optimality of problem \eqref{eqn:mpc_typical} can be expressed using its Lagrangian, which is given by
\begin{align}
\mathcal{L}\Big(\bu^\star_{[k,k+N-1]_d}\Big) &= \sum_{i=0}^{N-1} \ell(\bx_{k+i},\bu_{k+i})\nonumber\\
&+\sum_{j\in\ixactive\cup\ixinactive} \mu_j g_j(\bx)\nonumber\\
&+\sum_{j\in\iuactive\cup\iuinactive} \lambda_j h_j(\bu),
\label{eqn:langrangian_original}
\end{align}
where $\lambda_j\geq 0$ and $\mu_j\geq 0$. Given that the constraints indexed by $j \in\ixinactive $ for $g_j(\bx)$ and $j\in\iuinactive$ for $h_j(\bu)$ are assumed to be inactive, it follows that $\mu_j=0$ and $\lambda_j=0$  for $j \in \ixinactive$ and $j\in\iuinactive$ respectively. Therefore,
\begin{align}
\mathcal{L}\big(\bu^\star_{[k,k+N-1]_d}\big) & =\sum_{i=0}^{N-1} \ell(\bx_{k+i}, \bu_{k+i})\nonumber\\
&+\sum_{j\in\ixactive} \mu_j g_j(\bx)+\sum_{j\in\iuactive} \lambda_j h_j(\bu).
\label{eqn:langrangian_new}
\end{align}
Since \eqref{eqn:langrangian_original} and \eqref{eqn:langrangian_new} are identical, the KKT conditions based on $\mathcal{L}\left(\bu^\star_{[k,k+N-1]_d}\right)$ become the same for the reduced MPC problem \eqref{eqn:simplified_mpc_problem} and the actual MPC problem \eqref{eqn:mpc_typical}.
\end{proof}
\begin{remark}
    \normalfont The number of constraints that are reduced by solving \eqref{eqn:simplified_mpc_problem} instead of \eqref{eqn:mpc_typical} is $|\ixinactive\cup\iuinactive|$,{ where $|\mcl{S}|$ is the cardinality of the finite set $\mcl{S}$}. Note that as mentioned in Lemma \ref{lemma:1}, the optimal control sequence is not affected by solving the simplified MPC problem \eqref{eqn:simplified_mpc_problem}.
\end{remark}
\begin{remark}
    \normalfont {For a QP,} the computational complexity scales as $\mathcal{O}(M^3)$ with the number of constraints $M$ when solved using established interior point methods, assuming the number of variables is fixed \cite{nesterov1994interior}. Consequently, if the set of active constraints $\mathcal{C}_{\text{active}}$ is known a priori, the computational complexity of solving the QP reduces to $\mathcal{O}\left(|\ixinactive \cup \iuinactive|^3\right)$.
\end{remark}
}

\subsection{Learning framework for identification of active constraints\label{subsec:learning_framework}}
The dataset $\mathcal{D}$ corresponds to a specific instance of the MPC problem characterized by its parameters, i.e., $\mathcal{D}= \{\bx^i_0, \bx^{\textbf{ref},i}_k,\dots,\bx^{\textbf{ref},i}_{k+N}, \dots\}^{N_c}_{i=1} $ where the superscript $i$ denotes the $i^{\text{th}}$ parameter set. The corresponding labels for the output indicates the set of active constraints $\mathcal{C}^i_\text{active}$ with a value of $1$ and inactive constraints $ \mathcal{C}^i_\text{inactive}$ with a value of $0$. The transformer model is then trained on this dataset to learn a mapping from the input parameters to the set of active constraints. Mathematically, let $\mathcal{P}_c = \{\bx_0, \bx^{\textbf{ref}}_k,\dots,\bx^{\textbf{ref}}_{k+N}, \dots\}$ represent the input parameter for a given MPC problem instance.
The transformer model learns a function $\Pi^c_{\theta_c}$ such that $\Pi^c_{\theta_c}\left(\mathcal{P}_c\right) =\mathcal{C}_\text{active}$
where $\theta_c$ are parameters of the transformer model $\Pi^c_{\theta}$ \eqref{fig:training_c}. The transformer architecture \ref{fig:training_c} uses the attention mechanism to weigh the importance of different elements within $\mathcal{P}_c$ when predicting the active constraints. Specifically, the attention mechanism computes a set of attention weights $ \mathrm{A} \in \mathbb{R}^{n \times n} $ as follows:
\begin{align}
\mathrm{A} = \mathrm{softmax}\left(\frac{\mathrm{Q} \mathrm{K}^\mathrm{T}}{\sqrt{d_k}}\right),
\end{align}
where $ \mathrm{Q} \in \mathbb{R}^{n \times d_k} $ represents the query matrix, $ \mathrm{K} \in \mathbb{R}^{n \times d_k} $ the key matrix, and $ d_k $ is the dimensionality of the key vectors. The output of the attention mechanism, known as the context vector $ \mathrm{C} \in \mathbb{R}^{n \times d_v} $, is given by 
$\mathrm{C} = \mathrm{A} \mathrm{V}$
where $ \mathrm{V} \in \mathbb{R}^{n \times d_v} $ is the value matrix.
Each element of the context vector $\mathrm{C}$ is a weighted sum of the value vectors, with weights determined by the attention mechanism:
\begin{align}
\mathbf{c}_i = \sum_{j=1}^N \mathrm{softmax}\left(\frac{\langle\mathbf{q}_i, \mathbf{k}_j\rangle}{\sqrt{d_k}}\right) \cdot \mathbf{v}_j,
\end{align}
where $\langle\cdot,\cdot\rangle$ denotes the dot product, $\mathbf{q}_i$ is the $i^{\text{th}}$ row of $\mathrm{Q}$, $\mathbf{k}_j$ is the $j^{\text{th}}$ row of $\mathrm{K}$, $\mathbf{v}_j$ is the $j^{\text{th}}$ column of $\mathrm{V}$ and $\mathbf{c}_i$ is the $i^{\text{th}}$ row of $\mathrm{C}$. This process allows the model to weigh the importance of different elements within $\mathcal{P}_c$ when predicting the active constraints. The context vector $ \mathrm{C} $ captures the relevant information from the input parameters, $\mathcal{P}_c$, that is most relevant for predicting the active constraints. For training the transformer model, we use the Mean Squared Error (MSE) loss function that measures the difference between the predicted and actual active constraints.
Once the transformer model is trained, it can be used during online operation of MPC to predict the set of active constraints $ \mathcal{C}_\text{active} $ based on the current parameters, $\mathcal{P}_c$. This prediction allows the MPC solver to focus only on the active constraints, significantly reducing the computational complexity of the optimization problem. Constraints not predicted to be active, denoted as $ \mathcal{C}_\text{inactive}$, are removed from the problem formulation. Thus, the proposed approach streamlines the MPC process by predicting and removing inactive constraints in real-time, leading to faster computation of optimal control inputs without sacrificing constraint satisfaction.

{

\subsection{Transformer-Based Warm Start Initialization\label{subsec:transformer_warm_start}}
In addition to using transformers for inactive constraint removal, we leverage its capabilities to warm start the MPC problem as well. Towards this goal,
we use another transformer model $\Pi^w_{\theta_w}$ to predict an initial guess for the optimal control inputs $\tilde{\bu}_{[k,k+N-1]_d}$ $\left(\approx\bu^\star_{[k,k+N-1]_d}\right)$. Consequently, the state trajectories $ \{\bx_{k+1},\dots,\bx_{k+N}\} $ can also be computed, given the initial state $\bx_k$ at each time horizon $k$. By learning the optimal solutions of previous MPC problems, the transformer $\Pi^w_{\theta_w}$ can provide a starting point that is closer to the final solution.
Let $ \mathcal{Z}_{\mathrm{w}} = \{ \bx^i_k, \bx^{\textbf{ref},i}_k,\dots,\bx^{\textbf{ref},i}_{k+N}\}^{N_w}_{i=1} $ represent the set of current system states and reference trajectories. The transformer model $ \Pi^w_{\theta_w} $ \ref{fig:training_w} learns a mapping from $ \mathcal{Z}_{\mathrm{w}} $ to the warm start guess, i.e., $\Pi^w_{\theta_w}(\mathcal{Z}_{\mathrm{w}}) \rightarrow \tilde{\bu}_{[k,k+N-1]_d}$
where $ \tilde{\bu}_{[k,k+N-1]_d} $ is the approximate initial guess for the optimal control sequence (that is expected to be close to $\bu^\star_{[k,k+N-1]_d}$). This warm start solution is then used as the initial input for the MPC solver.
Note that the warm start provided by the transformer complements the online constraint removal strategy. After the constraints have been reduced learned model $\Pi^\star_{\theta_c}$, the MPC solver uses the transformer's warm start initialization to further speed up the convergence process. This dual approach ensures that the MPC problem is solved efficiently, both by reducing the number of inactive constraints as well as initiating the optimization process closer to the optimal solution via transformers.
}
\subsection{Execution phase\label{subsec:excetuation_phase}}
In the execution phase, given the set of parameters $ \mathcal{P}_c = \{\bx_0, \bx^{\textbf{ref}}_k,\dots,\bx^{\textbf{ref}}_{k+N}, \mathcal{C}\} $ that characterize an MPC problem, the trained transformer models, $\Pi^c_{\theta^\star_c}$ and $\Pi^w_{\theta^\star_w}$, are deployed in real-time to predict the active constraints, $\mathcal{C}_{\text{active}}$, and improve the warm start initialization, respectively. These active constraints are then added as constraints to the simplified MPC problem. The MPC problem is then solved iteratively via any state-of-the-art MPC solver as shown in Fig. \ref{fig:proposed_approach}. This reformulated problem is computationally less intensive than solving the complete MPC problem with all constraints. The optimal solution obtained from the simplified MPC problem is verified (see Fig. \ref{fig:proposed_approach}) to ensure it satisfies all the constraints of the original MPC. If any constraints are violated, the original MPC is solved to generate the optimal control inputs.

\subsection{Analytical solution to QP's after transformer based inactive constraint removal \label{subsec:qp}}
{In this section, we consider a special case of a general MPC problem with convex quadratic cost and linear dynamics, i.e., a QP.} Consider the following QP
\begin{subequations}
\begin{align}
\minimize\;\;\; &\frac{1}{2}\by^\mathrm{T}Q\by+\boldsymbol{p}^\mathrm{T}\by \\
\subjectto\;\;\; & A\by=\boldsymbol{b}, \quad C\by\leq \bd,
\end{align}
\label{eqn:all_constraints_qp}
\end{subequations}
where $\by$ is the augmented optimization variable obtained after converting the linear MPC into a QP \cite{bishop2024relu_qp}. If the active constraints are predicted accurately by the learned transformer model $\Pi^c_{\theta^\star_c}$ (where $\theta^\star_c$ are learned parameters), then the simplified QP with only equality constraints is given by
\begin{subequations}
\begin{align}
    \minimize\;\;\; & \frac{1}{2}\by^\mathrm{T}Q\by+\boldsymbol{p}^\mathrm{T}\by \\
    \subjectto \;\;& E\by=\boldsymbol{f},
\end{align}
\label{eqn:simplified_qp}
\end{subequations}
where $E=\mathrm{diag}(A,C_1)$, $\boldsymbol{f}=[\boldsymbol{b}^\mathrm{T},\;\bd^\mathrm{T}_1]^\mathrm{T}$, $C_1$ and $\bd_1$ are such that the inequality constraint $C\by\leq \bd$ is divided into active constraints, $C_1\by= \bd_1$, and inactive constraints, $C_2\by< \bd_2$. 
\begin{lemma}
 \normalfont Assuming that $Q$ is a positive definite (symmetric) matrix and $E$ a full row rank matrix, 
 the analytical solution to \eqref{eqn:simplified_qp} is given by
    \begin{align}
        \by^\star = Q^{-1} \left( E^\mathrm{T} \left( E Q^{-1} E^\mathrm{T} \right)^{-1} \left( E Q^{-1} \boldsymbol{p} + \boldsymbol{f} \right) - \boldsymbol{p} \right). \nonumber
    \end{align}
\end{lemma}
\begin{proof}
The Lagrangian for \eqref{eqn:simplified_qp} is given by:
\begin{align}
\mathcal{L}(\by, \boldsymbol{\lambda}) = \frac{1}{2} \by^\mathrm{T} Q \by + \boldsymbol{p}^\mathrm{T} \by + \boldsymbol{\lambda}^\mathrm{T} (E \by - \boldsymbol{f}),
\end{align}
where $\boldsymbol{\lambda}\geq 0$ is the vector of Lagrange multipliers associated with the equality constraints. At optimality, the gradient of $\mathcal{L}$ with respect to $\by$ and $\boldsymbol{\lambda}$ is zero i.e., $\nabla_{\by} \mathcal{L} = Q \by + \boldsymbol{p} + E^\mathrm{T} \boldsymbol{\lambda} = 0$ and $\nabla_{\boldsymbol{\lambda}} \mathcal{L} = E \by - \boldsymbol{f} = 0$, which can be written in a compact form as 
\begin{align}
    \begin{pmatrix}
Q & E^\mathrm{T} \\
E & 0
\end{pmatrix}
\begin{pmatrix}
\by \\
\boldsymbol{\lambda}
\end{pmatrix}
=
\begin{pmatrix}
-\boldsymbol{p} \\
\boldsymbol{f}
\end{pmatrix}.
\end{align}
The solution to this system provides the optimal $\by^\star$ and $\boldsymbol{\lambda}^\star$. Specifically, $\by^\star$ is given by
$\by^\star = -Q^{-1} \boldsymbol{p} - Q^{-1} E^\mathrm{T} \boldsymbol{\lambda}^\star$
where $\boldsymbol{\lambda}^\star$ is computed by solving
$E Q^{-1} E^\mathrm{T} \boldsymbol{\lambda} = E Q^{-1} \boldsymbol{p} + \boldsymbol{f}$.
Finally, after substituting $\boldsymbol{\lambda}^\star$ in the expression for $\by^\star$, we get
\begin{align}
\by^\star = Q^{-1} \left( E^\mathrm{T} \left( E Q^{-1} E^\mathrm{T} \right)^{-1} \left( E Q^{-1} \boldsymbol{p} + \boldsymbol{f} \right) - \boldsymbol{p} \right)\nonumber
\end{align}
Consequently, the result follows.
\end{proof}

\begin{remark}
If the optimal solution $ \by^\star $ satisfies all the constraints in the original QP \eqref{eqn:all_constraints_qp}, it will be applied to the robotic system. Otherwise, one must solve the original QP to generate the optimal control inputs. For large-scale QP problems, the matrix inversion $\left( E Q^{-1} E^\mathrm{T} \right)^{-1}$ can be further accelerated through the use of GPUs \cite{benner2013matrix_qp_gpu1,ezzatti2011high_qp_gpu2}.
\end{remark}

\subsection{Accelerating Nonlinear MPC (NMPC) problems after inactive constraint removal\label{subsec:acc_nonlinear_mpc}}
After the inactive constraints have been removed from \eqref{eqn:mpc_typical} by the learned model $\Pi^c_{\theta^\star_c}$, NMPC problems can be further accelerated by combining multiple nonlinear and non-convex constraints, \eqref{eqn:nonlin_cons_1} and \eqref{eqn:nonlin_cons_2}, into a single smooth constraint function using log-sum-exp expressions (a smooth approximation of union operation), i.e., $g^{\text{comb}}(\bx,\bu):=\text{log}\left(\sum_{j\in\ixactive}\mathrm{e}^{\beta g_j(\bx)}+\sum_{j\in\iuactive}\mathrm{e}^{\beta h_j(\bu)}\right)\leq 0$ for some $\beta>0$. By leveraging GPU parallelization (for instance, using $\texttt{cupy}$ \cite{nishino2017cupy} package) for the summation within the log-sum-exp function, significant runtime improvements can be achieved, particularly beneficial for real-time MPC applications. It can be shown that the constraint set $\mathcal{S}=\left\{(\bx,\bu)|\;g_i(\bx)\leq 0,\;h_j(\bu)\leq 0,\;i\in\ixactive,\;j\in\iuactive\right\}$ is a subset of $\mathcal{S}^c=\{(\bx,\bu)|\;g^{\text{comb}}(\bx,\bu)\leq 0\}$ and as $\beta$ tends to infinity, the set $\mathcal{S}$ tends to $\mathcal{S}^c$. Furthermore, this approach scales effectively with increasing the number of constraints, offering a promising solution for large-scale, nonlinear MPC problems that require rapid computation. Note that other functions, such as Mellowmax or p-Norm functions, can also be used instead of log-sum-exp functions.
\section{Results\label{sec:results}}
In this section, we compare our proposed approach, \ours, with recent baseline methods for solving MPC problems. Through our numerical experiments, we aim to answer the following questions $\textit{(i)}$ what is the average reduction in the number of inactive constraints observed using our approach? $\textit{(ii)}$ what is the overall decrease in computational time of \ours compared with the baseline methods? $\textit{(iii)}$ how do other approaches, such as Multi-Layer Perceptron (MLP), random forest \cite{breiman2001random_rf}, gradient boosting \cite{friedman2001greedy_gb} and Support Vector Machine (SVM) \cite{cortes1995support_svm} compare with this our proposed transformer architecture? 
The average computational time (averaged over $500$ MPC problems) for our approach is computed by $\alpha t_{\mathrm{RMPC}}+(1-\alpha)(t_{\mathrm{RMPC}}+t_{\mathrm{MPC}})$ where $\alpha\in[0,1]$, $t_{\mathrm{RMPC}}$, and $t_{\mathrm{MPC}}$ are the times taken to compute optimal control inputs for reduced MPC (RMPC) \eqref{eqn:simplified_mpc_problem} and MPC \eqref{eqn:mpc_typical} respectively. For instance, for $100$ MPC problem instances, if the learned transformer predicts the active constraints for $90$ problems correctly, then $\alpha=0.9,\;(90/100)$.  
We compare our proposed approach \ours with interior point methods \cite{andersen2020cvxopt_ip5,domahidi2013ecos_ip6,schwan2023piqp_ip8}, 
 active set methods \cite{ferreau2014qpoases_as1,hall2023highs_as3}, proximal interior method \cite{schwan2023piqp_ip8} 
as well as an augmented Lagrangian method \cite{bambade2022prox_al2}. {For our experiments, we consider three realistic MPC problems that are common in the robotics community. The first is that of balancing an Upkie wheeled biped robot ($15$ states and $7$ control inputs) \cite{bambade2023proxqp, qpbenchmark2024}. Second, we consider the problem of stabilizing a Crazyflie quadrotor ($12$ states and $4$ control inputs) from a random initial condition to a hovering position \cite{nguyen2024tiny_mpc}. Finally, we consider the problem of stabilizing an Atlas humanoid robot ($58$ states and $29$ control inputs) balancing on one foot \cite{bishop2024relu_qp}.}
All benchmarking experiments were performed on a desktop equipped with an Intel(R) Core(TM) i9-10900K CPU @ 3.70GHz and an NVIDIA RTX A4000 GPU with 16 GB of GDDR6 memory.

\subsection{Average percentage reduction in the number of inactive constraints}
Figure \ref{fig:const_reduction} illustrates the percentage reduction in the total number of inactive constraints for the three robotic systems considered. \ours achieves significant inactive constraint reductions across all systems, with an 89.0\% reduction for the wheeled biped, 93.6\% for the quadrotor, and 95.8\% for the humanoid. These results demonstrate the effectiveness of \ours in simplifying the MPC problem by removing inactive constraints, particularly for high-dimensional and complex systems such as the wheeled biped and the Atlas humanoid robot, thereby enhancing computational efficiency.
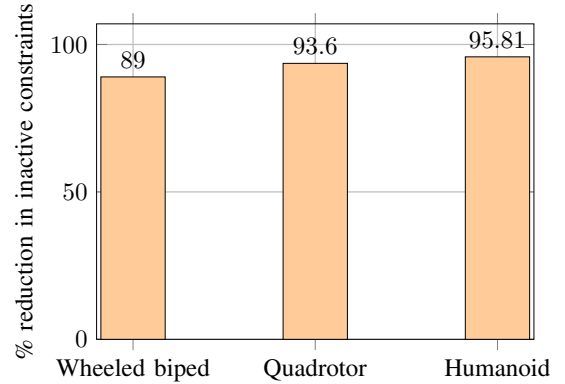
\begin{figure}[H]
    \centering
\begin{tikzpicture}[scale=0.95]
    \begin{axis}[
        ybar,
        grid=both,
        major grid style={line width=0.5pt,draw=gray!50},
        minor grid style={line width=0.2pt,draw=gray!20},
        symbolic x coords={Wheeled biped, Quadrotor, Humanoid},
        xtick=data,
        ylabel={\% reduction in inactive constraints},
        ymin=0,
        ymax=107,
        bar width=0.9cm,
        nodes near coords,
        nodes near coords align={vertical},
        x=2.55cm,
        width=16cm,
        height=6cm
    ]
     \addplot[
        ybar,
        fill=orange!40 
     ] coordinates {(Wheeled biped,89.0) (Quadrotor,93.6) (Humanoid,95.81) };
    
    \end{axis}
\end{tikzpicture}
    \caption{\small Average reduction in the total number of inactive constraints for the three robotic systems.}
    \label{fig:const_reduction}
\end{figure}
\subsection{Average reduction of computational time \label{sec:average_reduction_of_computational_time}}

Figures \ref{fig:random_mpc} and \ref{fig:quad_comp} illustrate the average computational time for various solvers when applied to three different systems: an Upkie wheeled biped robot, a Crazyflie quadrotor, and an Atlas humanoid robot. 
For the wheeled biped robot (see Fig. \ref{fig:random_mpc}), the \ours significantly reduces computational time across all solvers. Notably, for CVXOPT (in $10^{-2}$s), there is a reduction from $1.1549$ to $0.1698$ ($6.8\mathrm{x}$ improvement), and qpOASES (in $10^{-2}$s) from $1.3923$ to $0.1698$ ($8.19\mathrm{x}$ improvement). Even solvers with initially lower computational times, such as DAQP and ProxQP (in $10^{-2}$s), benefit from \ours, with reductions to $0.0174$ and $0.0227$, respectively. 

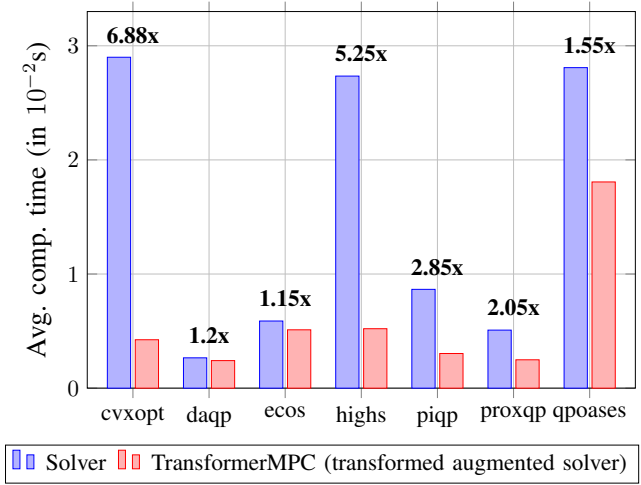
\begin{figure}[H]
    \centering
\begin{tikzpicture}[scale=0.88]
\begin{axis}[
    ybar,
    ymin=0,
    ymax=3.3,
        legend style={
        at={(-0.15,-0.15)},
        anchor=north west,
        legend columns=-1,
        column sep=0.3em 
    },
    ylabel={\large Avg. comp. time (in $10^{-2}$s)},
    symbolic x coords={ cvxopt, daqp, ecos, highs, piqp, proxqp, qpoases},
    xtick=data,
         x=1.15cm, 
             grid=both, 
    ]
\addplot coordinates {(cvxopt,2.8996) (daqp,0.2656) (ecos,0.5882) 
                      (highs,2.7345)  (piqp,0.8657) (proxqp,0.5085) 
                      (qpoases,2.80854) };
\addplot coordinates {(cvxopt,0.4241) (daqp,0.2413) (ecos,0.5111) 
                      (highs,0.5208) (piqp,0.3035) (proxqp,0.2484) 
                      (qpoases,1.80683) };
\node[above, yshift=-0.4cm, line width=1cm] at (axis cs:cvxopt,2.8996) {\textbf{\textcolor{black}{ 6.88x}}};
\node[above, yshift=-0.4cm, line width=1cm] at (axis cs:daqp,0.2656) {\textbf{\textcolor{black}{ 1.2x}}};
\node[above, yshift=-0.4cm, line width=1cm] at (axis cs:ecos,0.5882) {\textbf{\textcolor{black}{ 1.15x}}};

\node[above, yshift=-0.4cm, line width=1cm] at (axis cs:highs,2.7345) {\textbf{\textcolor{black}{ 5.25x}}};

\node[above, yshift=-0.4cm, line width=1cm] at (axis cs:piqp,0.8657) {\textbf{\textcolor{black}{ 2.85x}}};

\node[above, yshift=-0.4cm, line width=1cm] at (axis cs:proxqp,0.5085) {\textbf{\textcolor{black}{ 2.05x}}};
\node[above, yshift=-0.4cm, line width=1cm] at (axis cs:qpoases,2.80854) {\textbf{\textcolor{black}{ 1.55x}}};

\legend{Solver, TransformerMPC (transformed augmented solver)}
\end{axis}
\end{tikzpicture}



    \caption{\small Average computational time of proposed \ours compared with solvers for steering the quadrotor to a hover position from random initial conditions.}
    \label{fig:quad_comp}
\end{figure}
Furthermore, for the quadrotor (see Fig. \ref{fig:quad_comp}), \ours demonstrates consistent performance improvements across all solvers. CVXOPT's runtime (in $10^{-2}$s) decreases from $2.8996$ to $0.4241$ ($6.88\mathrm{x}$ improvement), while qpOASES (in $10^{-2}$s) shows a reduction from $2.80854$ to $1.80683$ ($1.55\mathrm{x}$ improvement). The most significant gains are observed in HiGHS (in $10^{-2}$s), where the runtime drops from $2.7345$ to $0.5208$. 
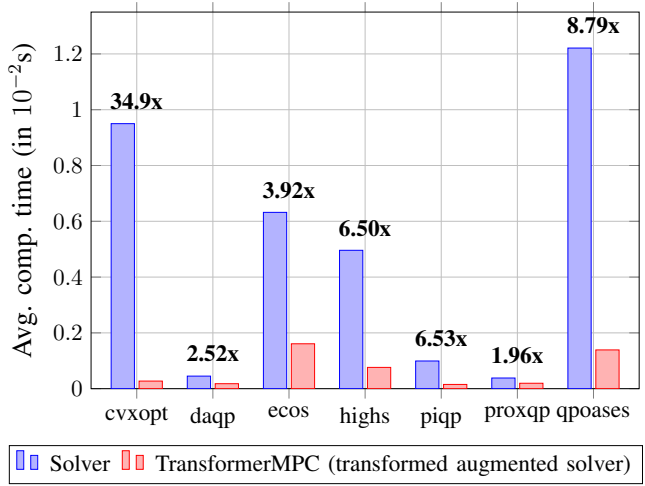
\begin{figure}[H]
    \centering
\begin{tikzpicture}[scale=0.88]
\begin{axis}[
    ybar,
    ymin=0,
    ymax=1.35,
        legend style={
        at={(-0.15,-0.15)},
        anchor=north west,
        legend columns=-1,
        column sep=0.3em 
    },
    ylabel={\large Avg. comp. time (in $10^{-2}$s)},
    symbolic x coords={ cvxopt, daqp, ecos, highs, piqp, proxqp, qpoases},
    xtick=data,
         x=1.15cm, 
             grid=both, 
    ]
\addplot coordinates {(cvxopt,0.95) (daqp,0.045) (ecos,0.632) 
                      (highs,0.496)  (piqp,0.0992) (proxqp,0.0382) 
                      (qpoases,1.221) };
\addplot coordinates {(cvxopt,0.0272) (daqp,0.0178) (ecos,0.1609) 
                      (highs,0.07625) (piqp,0.01518) (proxqp,0.01942) 
                      (qpoases,0.1389) };
\node[above, yshift=-0.4cm, line width=1cm] at (axis cs:cvxopt,0.95) {\textbf{\textcolor{black}{ 34.9x}}};
\node[above, yshift=-0.4cm, line width=1cm] at (axis cs:daqp,0.045) {\textbf{\textcolor{black}{ 2.52x}}};
\node[above, yshift=-0.4cm, line width=1cm] at (axis cs:ecos,0.632) {\textbf{\textcolor{black}{ 3.92x}}};

\node[above, yshift=-0.4cm, line width=1cm] at (axis cs:highs,0.496) {\textbf{\textcolor{black}{ 6.50x}}};

\node[above, yshift=-0.4cm, line width=1cm] at (axis cs:piqp,0.0992) {\textbf{\textcolor{black}{ 6.53x}}};

\node[above, yshift=-0.4cm, line width=1cm] at (axis cs:proxqp,0.0382) {\textbf{\textcolor{black}{ 1.96x}}};
\node[above, yshift=-0.4cm, line width=1cm] at (axis cs:qpoases,1.221) {\textbf{\textcolor{black}{ 8.79x}}};

\legend{Solver, TransformerMPC (transformed augmented solver)}
\end{axis}
\end{tikzpicture}



    \caption{\small Average computational time of proposed \ours on balancing of Atlas humanoid robot on one foot compared with recent baseline methods}
    \label{fig:humanoid_comp}
\end{figure}
Finally, for the problem of Atlas balancing on one foot (see Fig. \ref{fig:humanoid_comp}), CVXOPT runtime (in $10^{-2}$s) has a reduction from $0.95$ to $0.0272$ ($34.9\mathrm{x}$ improvement), while qpOASES (in $10^{-2}$s) sees an $8.79\mathrm{x}$ reduction from $1.221$ to $0.1389$. Other solvers like PIQP and HiGHS also demonstrate substantial improvements of $6.53\mathrm{x}$ and $6.50\mathrm{x}$, respectively. Even for lower-time solvers like DAQP and ProxQP, \ours achieves notable gains, reducing the time by $2.52\mathrm{x}$ and $1.96\mathrm{x}$, respectively.
This reduction indicates that \ours effectively enhances solver performance by warm-starting and removing inactive constraints, making it real-time implementable for robotic systems such as wheeled biped robots, quadrotors, and humanoid robots where computational efficiency is critical.
\subsection{\small Comparison with other learned models for active constraint prediction\label{subsec:comparison_with_nn_architecutres}}
For these numerical experiments, we employed an $80-20$ train-test split of the dataset, using $80\%$ of the data for training the models and $20\%$ for testing. Fig. \ref{fig:other_nn} presents the prediction accuracy on test data for different learning models in predicting inactive constraints across the three robotic systems. Among the models, the learned transformer model $\Pi^c_{\theta^\star_c}$ consistently achieves the highest accuracy, with $89\%$ for wheeled biped robot, $92\%$ for quadrotor, and $95.8\%$ for humanoid. In contrast, other learned models, including MLP, Random Forest \cite{breiman2001random_rf}, Gradient Boosting \cite{friedman2001greedy_gb}, Support Vector Machine (SVM) \cite{cortes1995support_svm}, and Logistic Regression, show significantly lower accuracy across all systems. These results highlight the efficacy of the learned transformer model in accurately predicting inactive constraints.
\begin{figure}[ht]
    \centering
\begin{tikzpicture}[scale=0.78]
\begin{axis}[
    ybar,
     enlargelimits=0.19,
      legend style={
        at={(0.02,-0.25)},
        anchor=north west,
        legend columns=3,
        column sep=1em 
    },
    ylabel={\large \% Accuracy on test data},
    symbolic x coords={wheeled biped,quadrotor, humanoid},
    xtick=data,
            ymin=13.93,
        ymax=90,
        bar width=0.36cm,
    nodes near coords,
    nodes near coords align={vertical},
        x=3.5cm,
grid=both,
            width=5.5cm,
         height=7cm
    ]
\addplot coordinates {(wheeled biped,89) (quadrotor,92.0) (humanoid,95.8)};
\addplot coordinates {(wheeled biped,3.3) (quadrotor,0.5) (humanoid,12.5)};
\addplot coordinates {(wheeled biped,3.3) (quadrotor,2) (humanoid,4.5)};
\addplot coordinates {(wheeled biped,8.3) (quadrotor,18) (humanoid,12.5)};
\addplot coordinates {(wheeled biped,15.8) (quadrotor,35.5) (humanoid,41.6)};
\addplot coordinates {(wheeled biped,10) (quadrotor,43.5) (humanoid,42.9)};

\legend{Transformer (Ours),MLP,Random forest, Gradient boosting, SVM, Logistic regression}
\end{axis}
\end{tikzpicture}
    \caption{\small Percentage accuracy for inactive constraint removal via different learned models on test data.
    }
    \label{fig:other_nn}
\end{figure}
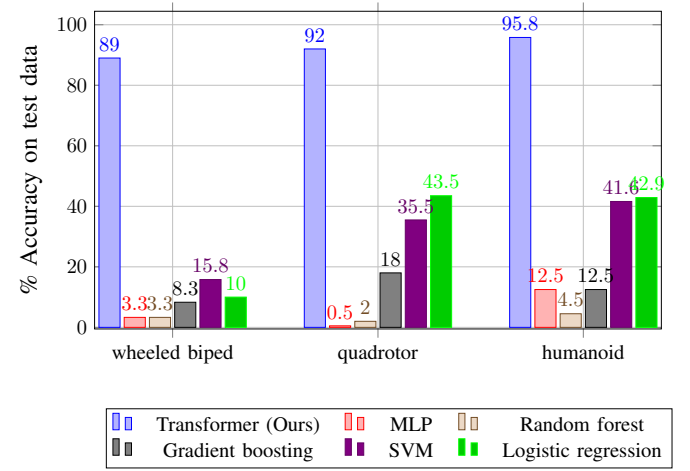

\section{Conclusions\label{sec:conclusion}}
In this paper, we addressed the problem of improving the computational efficiency in Model Predictive Control (MPC) by introducing a transformer-based approach for online inactive constraint removal as well as warm start initialization. Our approach reduced the computational burden by focusing on a subset of constraints predicted by the learned transformer model. We synthesized optimal control inputs by solving the MPC problem with this reduced set of constraints and verified the efficacy of our approach with complex robotic systems, demonstrating significant improvements in performance and feasibility for real-time applications over other state-of-the-art MPC solvers. Future work includes the implementation of TransformerMPC on resource constrained robotic systems as well as extending our approach to multi-agent systems.

\bibliography{main.bib}





\end{document}